\documentclass[letterpaper]{article}
\newif\ifdraft\drafttrue
\newif\ifinlineref\inlinereffalse
\newif\iffinal\finalfalse
\newif\ifextended\extendedfalse
\newif\ifdotikz\dotikzfalse

\draftfalse
\iffinal\else
\fi
\ifdraft
\newcommand{\comment}[1]{{\bf\color{blue}{*** #1 ***}}}
\else
\newcommand{\comment}[1]{}
\fi

\usepackage{aaai16}
\usepackage{times}
\usepackage{helvet}
\usepackage{courier}

\usepackage{amsthm}
\usepackage{amsmath}
\usepackage{amsfonts}
\usepackage{amssymb}
\usepackage{todonotes}
\usepackage{enumerate}
\usepackage{soul}
\usepackage{subfigure}
\usepackage{color}

\usepackage{microtype}

\newcommand{\finalcomment}[1]{}

\title{Reactive Policies with Planning\\ for Action Languages 
\thanks{This work has been supported by Austrian Science Fund (FWF) project W1255-N23.}}
\author{Zeynep G. Saribatur \and Thomas Eiter\\ 
Institut f\"ur Informationssysteme, Technische Universit\"at Wien\\
Favoritenstra\ss e 9-11, A-1040 Vienna, Austria\\ 
\texttt{\{zeynep,eiter\}@kr.tuwien.ac.at}
}
\def\ba{\begin{array}}
\def\ea{\end{array}}
\def\be{\begin{enumerate}}
\def\ee{\end{enumerate}}
\def\bi{\begin{itemize}}
\def\ei{\end{itemize}}
\def\beq{\begin{equation}}
\def\eeq#1{\label{#1}\end{equation}}
\def\beeq{\begin{equation*}}
\def\eeeq{\end{equation*}}

\def\eqs{\,{=}\,}

\def\leqs{\,{\leq}\,}
\def\geqs{\,{\geq}\,}

\def\ins{\,{\in}\,}

\def\nmodels{\,{\nvDash}\,}

\def\mi#1{\mathit{#1\/}}

\def\hatbf#1{\widehat{\textbf{#1\/}}}
\theoremstyle{theorem}
\newtheorem{prop}{Proposition}
\newtheorem{cor}{Corollary}
\newtheorem{thm}{Theorem}

\theoremstyle{definition}
\newtheorem{defn}{Definition}
\newtheorem{exmp}{Example}

\usepackage{graphicx}

\newcommand{\Reach}{\ensuremath{\mathit{Reach}}}
\newcommand{\Res}{\ensuremath{\mathit{Res}}}

\newcommand{\caused}[2]{\ensuremath{{{\textrm{\bf caused\ }}\mi{#1}
      {\textrm{\bf\ if\ }} \mi{#2}}}}

\begin{document}
\maketitle

\begin{abstract}
We describe a representation in a high-level transition system for policies that express a reactive behavior for the agent. We consider a target decision component that figures out what to do
next and an (online) planning capability to compute the plans needed to
reach these targets. Our representation allows one to analyze the flow
of executing the given reactive policy, and to determine whether it
works as expected. Additionally, the flexibility of the representation
opens a range of possibilities for designing behaviors.
\end{abstract}

%Agents are autonomous systems where they
Autonomous agents are systems that decide for themselves what to do to
satisfy their design objectives. These agents have a knowledge base
that describes their capabilities, represents facts about the world
and helps them in reasoning about their course of actions. A reactive agent
interacts with its environment. It perceives the current state of the
world through sensors, consults its memory (if there is any), reasons
about actions to take and executes them in the environment. A policy
for these agents gives guidelines to follow during their interaction
with the environment.

As autonomous systems become more common in our lives, the issue of
verifying that they behave as intended becomes more important. During
the operation of an agent, one would want to be sure that by following
the designed policy, the agent will achieve the desired results. It
would be highly costly, time consuming and sometimes even fatal to
realize at runtime that the designed policy of the agent does not
provide the expected properties.

For example, in search and rescue scenarios, an agent needs to find a missing
person 
%usually 
in unknown environments. A naive approach would be to directly try to find a plan that achieves the main goal of finding the person. However, this problem easily becomes troublesome, since not knowing the environment causes the planner to consider all possible cases and find a plan that guarantees reaching the goal in all settings. Alternatively, one can describe a reactive policy for the agent that determines its course of actions according to its current knowledge, and guides the agent in the environment towards the main goal. A possible such policy could be ``always
% going 
move to the farthest unvisited point in
visible distance, until a person is found". Following this reactive
policy, the agent would traverse the environment by choosing its
actions to reach the farthest possible point from the current state,
and by reiterating the decision process after reaching a new
state. The agent may also remember the locations it has been in and
gain information (e.g. obstacle locations) through its sensors on the
way. Verifying beforehand whether or not the designed policy of the agent satisfies the desired goal (e.g. can the agent always
find the person?), in all possible instances of the environment is
nontrivial.

Action languages \cite{gelfondaction98} provide a useful framework on defining actions and reasoning about them, by modeling dynamic systems as transition systems. Their declarative property helps in describing the system in an understandable, concise language, and they also address the problems encountered when reasoning about actions. %such as the frame problem and the ramification problem. %These languages can be extended to describe reactive agents and define policies. 
By design, these languages are made to be decidable, which ensures reliable descriptions of dynamic systems. As these languages are closely related with classical logic and answer set programming (ASP) \cite{whatasp08,lifschitz08}, they can be translated into logic programs and queried for computation. The programs produced by such translations can yield sound and complete answers to such queries. There have been various works on action languages \cite{gelfondaction98,gelfondlifschitz93,c98} and their reasoning systems \cite{giunchiglia2004nct,coala10}, %, that are able to answer queries over action domain descriptions represented in action languages. 
with underlying mechanisms % of these reasoning systems
that rely on SAT and ASP solvers. %Action languages can be used to describe reactive agents and define policies for them. 

The shortage of representations that are capable of modeling reactive policies prevents one from verifying such policies using action languages as above before putting them into use. The necessity of such a verification capability motivates us to address this issue. We thus aim for a general model that allows for verifying the reactive behavior of agents in environments with different types in terms of observability and determinism. In that model, we want to use the representation power of the transition systems described by action languages and combine components that are efficient for describing reactivity.

Towards this aim, we consider in this paper agents with a reactive behavior that decide their
course of actions by determining targets to achieve during their interaction with the
environment. Such agents come with an (online) planning
capability that computes plans to reach the targets.
This method matches the observe-think-act cycle of Kowalski and Sadri \shortcite{DBLP:journals/amai/KowalskiS99}, but involves a planner that considers targets. The flexibility
in the two components - target development and external planning -
allow for a range of possibilities for designing behaviors. For
example, one can use HEX \cite{hex05} to describe a program that
determines a target given the current state of an agent, finds the
respective plan and the execution schedule. ACTHEX programs \cite{acthex13}, in particular,
provide the tools to define such reactive behaviors as it allows for
iterative evaluation of the logic programs and the ability to observe
the outcomes of executing the actions in the environment. Specifically, we make the following contributions:

\begin{enumerate}[(1)]
\item We introduce a novel framework for describing the semantics of a policy that follows a reactive behavior, by integrating components of target establishment and online planning. The purpose of this work is not synthesis, but to lay foundations for verification of behaviors of (human-designed) reactive policies. The outsourced planning might also lend itself for modular, hierarchic planning, where macro actions (expressed as targets) are
turned into a plan of micro actions. Furthermore, outsourced planning
may also be exploited to abstract from correct sub-behaviors (e.g.
going always to the farthest point).
\item We relate this to action languages and discuss possibilities for policy formulation. In particular, we consider the action language $\mathcal{C}$ \cite{c98} to illustrate an application.
\end{enumerate}
%%%%%%%%%%\\

The remainder of this paper is organized as follows. After some
preliminaries, we present a running example and  then the general framework
for modeling policies with planning. After that, we consider the
relation to action languages, and as a particular application we
consider (a fragment of) the action language $\mathcal{C}$. We briefly
discuss some related work and conclude with some issues for ongoing
and future work.

\section{Preliminaries}

\begin{defn}
A \emph{transition system} $\mathcal{T}$ is defined as $\mathcal{T}=\langle S,S_0,\mathcal{A},\Phi\rangle$ where
\bi
\item $S$ is the set of states.
\item $S_0 \subseteq S$ is the set of possible initial states.
\item $\mathcal{A}$ is the set of possible actions.
\item $\Phi: S \times \mathcal{A} \rightarrow 2^{S} $ is the transition function, returns the set of possible successor states after applying a possible action in the current state.
\ei

For any states $s,s' \in S$, we say that there is a \emph{trajectory} between $s$ and $s'$, denoted by $s \rightarrow^\sigma s'$ for some action sequence $\sigma = a_1,\dots,a_n$ where $n \geq 0$, if there exist $s_0,\dots,s_n \in S$ such that $s=s_0,s'=s_n$ and $s_{i+1} \in \Phi(s_i,a_{i+1})$ for all $0 \leq i < n$.
\label{defn:TS}
\end{defn}

We will refer to this transition system as the \emph{original} transition system. The constituents $S$ and $\mathcal{A}$ are assumed to be finite in the rest of the paper. Note that, this transition system represents fully observable settings. Large environments cause high number of possibilities for states, which cause the transition systems to be large. Especially, if the environment is nondeterministic, the resulting transition system contains high amount of transitions between states, since one needs to consider all possible outcomes of executing an action.

\subsection{Action Languages}

Action languages describe a particular type of transition systems that are based on action signatures.
An \emph{action signature} consists of a set \textbf{V} of value names, a set \textbf{F} of fluent names and a set \textbf{A} of action names. Any \emph{fluent} has a \emph{value} in any \emph{state of the world}. %Execution of an \emph{action} leads to reach from one state to a \emph{successor} state. 
 %nondeterminism of an action or nondeterminism in the environment.
 
A transition system of an action signature $\langle \textbf{V},\textbf{F},\textbf{A} \rangle$ is similar to Defn.~\ref{defn:TS}, where $\mathcal{A}=\textbf{A}$ and $\Phi$ corresponds to the relation $R \subseteq S \times \textbf{A} \times S$. In addition, we have a value function $V: \textbf{F} \times S \rightarrow \textbf{V}$, where $V(P,s)$ shows the \emph{value of $P$} in state $s$. A transition system can be thought as a labeled directed graph, where a state $s$ is represented by a vertex labeled with the function $P \rightarrow V(P,s)$, that gives the value of the fluents. Every triple $\langle s, a, s'\rangle \in R$ is represented by an edge leading from a state $s$ to a state $s'$ and labeled by $a$.

An action $a$ is \emph{executable} at a state $s$, if there is at least one state $s'$ such that $\langle s, a, s'\rangle \in R$ and $a$ is \emph{deterministic} if there is at most one such state. %The successor state after execution of an action doesn't need to be unique, which can represent nondeterminism.
Concurrent execution of actions can be defined by considering transitions in the form $\langle s,A,s'\rangle$ with a set $A \subseteq \textbf{A}$ of actions, where each action $a \ins A$ is executable at $s$.
%\end{defn}
%Concurrent execution of actions can also be defined by considering the action names as truth-valued functions on a set $\textbf{E}$ of ``elementary action names": $\textbf{A}{=}\{\textsf{f,t}\}^\textbf{E}$. Executing an action $A$ means to execute concurrently all elementary actions such that $A(E){=}\textsf{t}$ and all other elementary action names are mapped to $\textsf{f}.$

An action signature $\langle \textbf{V},\textbf{F},\textbf{A} \rangle$ is \emph{propositional} if its value names are truth values: $\textbf{V} {=} \{\mathsf{f,t}\}$. 
%Note that in
In this work, we 
%consider
confine to propositional action signatures.

The transition system allows one to answer queries about the program. For example, one can find a plan to reach a goal state from an initial state, by searching for a path between the vertices that represent these states in the transition system. One can express properties about the paths of the transition system by using an action query language.

\section{Running Example: Search Scenarios}

Consider a memoryless agent that can sense horizontally and
vertically, in an unknown $n{\times} n$ grid cell environment with
obstacles, where a missing person needs to be found. Suppose we are
given the action description of the agent with a policy of ``always
going to the farthest reachable point in visible distance (until a
person is found)". Following this reactive policy, the agent chooses its course of actions to reach the farthest
reachable point, referred as \emph{target}, from its current location with respect to its current knowledge about the environment. After executing the plan and reaching a state that satisfies the target, the decision process is reiterated and a new target, hence a new course of actions, is determined.

Given such a policy, one would want to check whether or not the agent can always find the person, in all instances of the environment. 
Note that we assume that the obstacles are placed in a way that the person is always reachable.

\begin{figure}[h!]
\centering
%\captionsetup[subfigure]{justification=centering}
\subfigure[]{\includegraphics[scale=0.8]{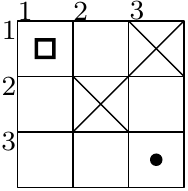}}\qquad
\subfigure[]{\includegraphics[scale=0.8]{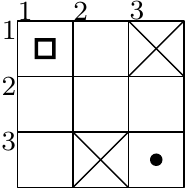}}\qquad
\subfigure[]{\includegraphics[scale=0.8]{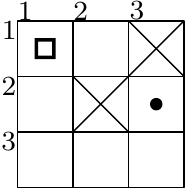}}\qquad
\caption{Possible instances of a search scenario}
\label{fig:simple}
\end{figure}

Figure~\ref{fig:simple} shows some possible instances for $n{=}3$,
where the square in a cell represents the agent and 
%the disk
the dot represents the missing person. The course of actions determined by the policy in all the instances is to move to (3,1), which is the farthest reachable point, i.e. target. It can be seen that (a) is an
instance where the person can be found with the given policy, while in
(b) the agent goes in a loop and can't find the person, since after reaching (3,1) it will decide to move to (1,1) again. In (c), after
reaching (3,1) following its policy, the agent has two
possible directions to choose, since there are two farthest points. It
can either move to (3,3), which would result in seeing the person, or it
can move back to (1,3), which would mean that there is
a possibility for the agent to go in a loop.
%\label{ex:simple}

Notice that our aim is different from \emph{finding a policy} (i.e. global plan) that satisfies certain properties (i.e. goals) in an unknown environment. %, which is generally a more difficult task.For the above example, one may confuse our aim as finding a global plan where the robot can always find the person in an unknown environment, and it can be easily seen that this is a difficult task. 
On the contrary, we assume that we are given a representation of a system with a certain policy, and we want to check what it is capable (or incapable) of.

\section{Modeling Policies in Transition Systems}

A reactive policy is described to reach some main goal, by guiding the agent through its interaction with the environment. This guidance can be done by determining the course of actions to bring about targets from the current situation, via externally computed plans. A transition system that models such policies should represent the flow of executing the policy, which is the agent's actual trajectory in the environment following the policy. This would allow for verifying whether execution of a policy results in reaching the desired main goal, i.e. the policy works.

We define such a transition system by clustering the states
into groups depending on a \emph{profile}. A profile is determined by evaluating a set of formulas
over a state that 
% to identify their attributes,
informally yield attribute (respectively feature) values; states 
% that provide the same evaluation
with the same attribute values are clustered into one. The choice of formulas for determining profiles depends on the given policy or the environment one is considering. Then, the
transitions between these clusters are defined according to the policy.
The newly defined transitions
are able to show the evaluation of the policy by a higher level action from
one state to the next state. This next state satisfies the target determined by a \emph{target component}, and the higher level action
corresponds to the execution of an externally computed plan. 

Having such a classification on states and defining higher level
transitions between the states can help in reducing the state space or
the number of transitions when compared to the original transition
system. Furthermore, it aids
in abstraction and allows one to emulate a modular hierarchic
approach, in which a higher level (macro) action, expressed by a
target, is realized in terms of a sequence of (micro) actions that
is compiled by the external planner, which may use different
ways (planning on the fly, resorting to scripts etc.)

\subsection{State profiles according to the policy}

We now describe a classification of states, which helps to omit parts of the state that are irrelevant with respect to the environment or the policy. This classification is done by determining profiles, and clustering the states accordingly.

\begin{exmp}
\begin{figure}[t]
\centering
\includegraphics[scale=0.8]{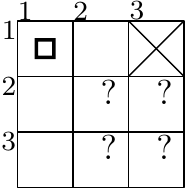}
\caption{A cluster of states}
\label{fig:simpleex}
\end{figure}

Remember the possible instances from Figure~\ref{fig:simple} in the
running example. Due to partial observability, the agent is unable to
distinguish the states that it is in, and the unobservable parts are
irrelevant to the policy. Now assume that there are fluents that hold
the information of the agent's location, the locations of the
obstacles and the reachable points. One can determine a profile of the
form ``the agent is at (1,1), sees an obstacle at (1,3), and is able
to reach to points at (1,2), (2,1), (3,1)" by not considering the
remaining part of the environment that the agent can not observe. The
states that have this profile can be clustered in one group as in
Figure~\ref{fig:simpleex}, where the cells with question marks
demonstrate that they are not observable by the agent.
\end{exmp}

For partially observable environments, the notion of indistinguishable states can be used in the classification of states. The states that provide the same observations for the agent are considered as having the same profile.
However, in fully observable environments, observability won't help in reducing the state space. One needs to find other notions to determine profiles.

We consider a \emph{classification function}, $h : S \rightarrow \Omega_h$, where $\Omega_h$ is the set of possible state clusters. This is a general notion applicable to fully and partially observable cases. %In partial observable settings $h$ will act as the observation function described above.

\begin{defn}
An \emph{equalized state} relative to the classification function $h$ is a state $\hat{s} \in \Omega_h$. The term \emph{equalized} comes from the fact that the states in the same classification are considered as the same, i.e. equal.

To talk about a state $s$ that is clustered into an equalized state $\hat{s}$, we use the notation $s \ins \hat{s}$, where we identify $\hat{s}$ with its pre-image (i.e. the set of states that are mapped to $\hat{s}$ according to $h$).
\end{defn}

Different from the work by Son and Baral \shortcite{baral01} where
they consider a ``combined-state" which consists of the real state of
the world and the states that the agent thinks it may be in, we
consider a version where we combine the real states into one state if
they provide the same classification (or observation, in case of
partial observability) for the
agent. %The number of states by our definition would be less or equal to the number of real states.
The equalization of states allows for omitting the details that are irrelevant to the behavior of the agent.

\subsection{Transition systems according to the policy}

We now define the notion of a transition system that is able to represent the evaluation of the policy on the state clusters.

%%%%%%%%%%%%%%%%%%%%%%main definition begin%%%%%%%%%%%%%%%%
\begin{defn}
An \emph{equalized (higher level) transition system $\mathcal{T}_h$}, with respect to the classification function $h$, is defined as $\mathcal{T}_h= \langle \widehat{S}, \widehat{S}_0,G_B, \mathcal{B}, \Phi_{\mathcal{B}} \rangle$, where
\bi
\item $\widehat{S}$ is the finite set of equalized states; 
%\item $\widehat{S}_0 \subseteq \widehat{S}$ is the finite set of possible initial states in the equalized states.
\item $\widehat{S}_0\subseteq \widehat{S}$ is the finite set of initial equalized states, where $\hat{s} \in \widehat{S}_0$ if there is some $s_i \in \hat{s}$ such that $s_i \in S_0$ holds;

\item $G_B$ is the finite set of possible \emph{targets} relative to the behavior, where a target can be satisfied by more than one equalized state;
\item $\mathcal{B} : \widehat{S} \rightarrow 2^{G_B}$, is the \emph{target function} that returns the possible targets to achieve from the current equalized state, according to the policy;
\item $\Phi_{\mathcal{B}}: \widehat{S} \rightarrow 2^{\widehat{S}} $ is the transition function according to the policy, referred to as the \emph{policy execution function}, returns the possible resulting equalized states after applying the policy in the current equalized state.
\ei

The target function gets the equalized state as input and produces the possible targets to achieve. These targets may be expressed as formulas over the states (in particular, of states that are represented by fluents or state variables), or in some other representation. 
A target can be considered as a {subgoal condition} to hold at the follow-up state, depending on the current equalized state. The aim would be to intend to reach a state that satisfies the conditions of the target, without paying attention to the steps taken in between. That's where the policy execution function comes into the picture.

The formal description of the policy execution function is as follows:
\beeq \ba {rl}
\Phi_B(\hat{s}) = \{ \hat{s}' ~|~ \hspace{-1em}&\hat{s}' \in \Res(\hat{s},\sigma), \\
& \sigma \in \Reach(\hat{s},g_B),g_B \in \mathcal{B}(\hat{s})\},
\ea \eeeq
where $\Reach$ is an outsourced function that returns a plan $\sigma = \langle a_1,\dots,a_n\rangle, n\geq 0$ needed to reach a state that meets the conditions $g_B$ from the current equalized state $\hat{s}$:
$$\Reach(\hat{s},g_B)\subseteq \{\sigma ~|~ \forall \hat{s}' \in \Res(\hat{s},\sigma) : \hat{s}' \models g_B\}$$
where $\hat{s} {\models} g_B \Leftrightarrow \forall s \ins \hat{s} : s {\models} g_B,$
and $\Res$ gives the resulting states of executing a sequence of actions at a state $\hat{s}$:
\beeq \ba {ll}
\Res(\hat{s}, \langle a_1,\dots,a_{n\geq 1} \rangle)&= \\
& \hspace{-1.3in} \left\{ \ba {lcl}
\bigcup_{\hat{s}'\in \hat{\Phi}(\hat{s},a_1)} \Res(\hat{s}',\langle a_2,\dots,a_{n} \rangle) & \mbox{ } & \hat{\Phi}(\hat{s},a_1){\neq} \emptyset\\
\{\hat{s}_{err}\} & \mbox{ } & \hat{\Phi}(\hat{s},a_1){=} \emptyset
\ea \right.\\
\\
\Res(\hat{s},\langle \rangle)&= \{\hat{s}\} 
\ea \eeeq 
where the state $\hat{s}_{err}$ is an artifact state that does not satisfy any of the targets, and
$$\hat{\Phi}(\hat{s},a)=\{ \hat{s}' ~|~ \exists s' \in \hat{s}' ~\exists s \in \hat{s}: s' \in \Phi(s,a)\}.$$
\label{defn:main}

\vspace{-.75\baselineskip}
\end{defn}
\begin{figure}[t]
\centering
\includegraphics[scale=0.7]{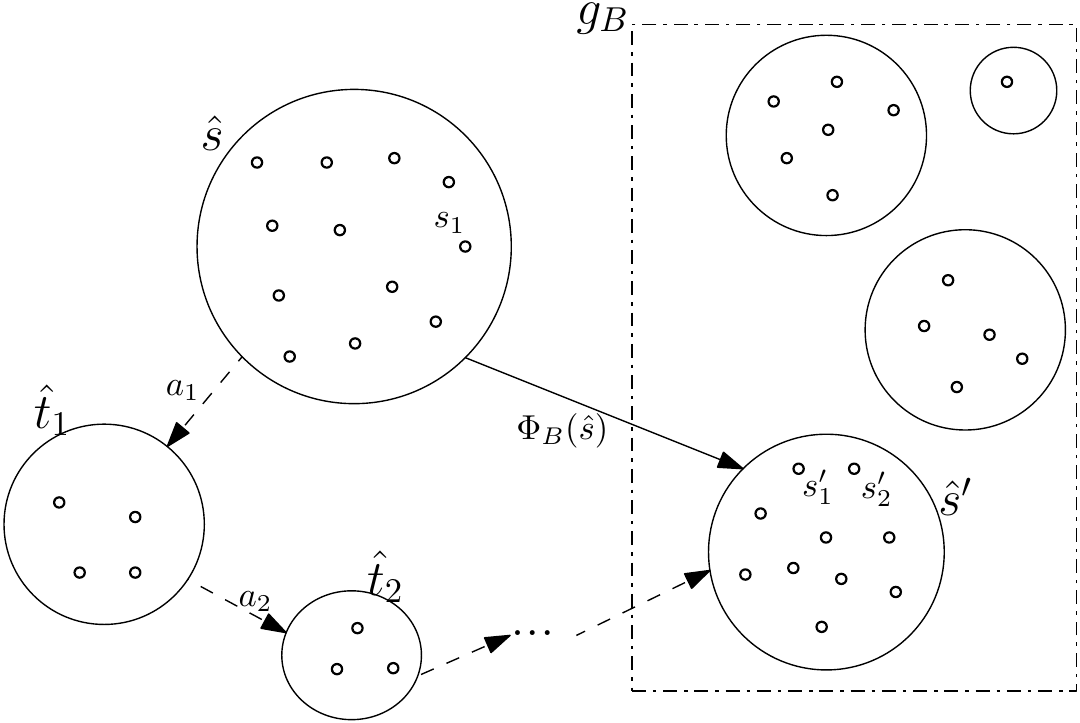}
\caption{A transition in the equalized transition system}
\label{fig:trans}
\end{figure}

Figure~\ref{fig:trans} demonstrates a transition in the equalized transition system. The equalized states may contain more than one state that has the same profile. Depending on the current state, $\hat{s}$, the policy chooses the next target, $g_B$, that should be satisfied. There may be more than one equalized state satisfying the same target. The policy execution function $\Phi_B(\hat{s})$ finds a transition into one of these equalized states, $\hat{s}'$, that is reachable from the current equalized state. The transition $\Phi_B$ is considered as a big jump between states, where the actions taken and the states passed in between are omitted.

Notice that we assume that the outsourced $\Reach$ function is able to return conformant plans that guarantee to reach a state that satisfies the determined targets. In particular, $\sigma$ may also contain only one action. For practical reasons, we consider $\Reach$ to be able to return a subset of all conformant plans. The maximal possible $\Reach$, where we have equality, is denoted with $\Reach_0$.

Consider the case of uncertainty, where the agent requires to do some action, e.g. $\mathit{checkDoor}$, in order to get further information about its state. One can define the target function to return as target a dummy fluent to ensure that the action is made, e.g. $\mathit{doorIsChecked}$, and given this target, the $\Reach$ function can return the desired action as the plan. The nondeterminism or partial observability of the environment is modeled through the set of possible successor states returned by $\Res$.

The generic definition of the equalized transition system allows for the possibility of representing well-known concepts like purely reactive systems or universal planning \cite{cimattiobdd98}. To represent reactive systems, one can describe a policy of ``pick some action". This way one can model reactive systems that do not do reasoning, but immediately react to the environment with an action. As for the exactly opposite case, which is finding a plan that guarantees reaching the goal, one can choose the target as the main goal. Then, the $Reach$ would have the difficult task of finding a universal plan or a conformant plan that reaches the main goal. If however, one is aware of such a plan, then it is possible to mimic the plan by modifying the targets $G_B$ and the target function $\mathcal{B}$ in a way that at each point in time the next action in the plan is returned by $Reach$, and the corresponding transition is made. For that, one needs to record information in the states and keep track of the targets.

As the function $\Reach$ is
outsourced, we rely on an implementation that returns conformant plans
for achieving transitions in the equalized transition systems. This
naturally raises the issue of whether a given such implementation is
suitable, and leads to the question of soundness (only correct plans
are output) and completeness (some plan will be output, if one
exists). We next assess how expensive it is to test this, under some
assumptions about the representation and computational properties of
(equalized) transition systems, which will then also be used for
assessing the cost of policy checking.

\paragraph{Assumptions} We have certain assumptions on the
representation of the system. We assume that given a state $s \in S$
which is implicitly given using a binary encoding, the cost of
evaluating the classification $h(s)$, the (original) transition
$\Phi(s,a)$ for some action $a$, and recognizing the initial state,
say with $\Phi_{init}(s)$, is polynomial. The cost could also be in
NP, if projective (i.e. existentially quantified) variables are allowed. Furthermore, we assume that
the size of the representation of a ``target'' in $G_B$ is polynomial in size of the state, so that given a string, one can check in polynomial
time if it is a correct target description
$g_B$. %Then one can test whether $s \models g_B$ in polynomial time.
This test can also be relaxed to be in NP by allowing projective variables. 

Given these assumptions, we have the following two
results. These results show the cost of checking whether an
implementation of $\Reach$ that we have at hand is sound (delivers
correct plans) and in case does not skip plans (is complete); we
assume here that testing whether $\sigma \in \Reach(\hat{s},g_B)$ is
feasible in $\Pi^p_2$ (this is the cost of verifying conformant plans,
and we may assume that $\Reach$ is no worse than a naive guess and
check algorithm).

\begin{thm}[soundness of $\Reach$]
Let $\mathcal{T}_h= \langle \widehat{S}, \widehat{S}_0,$ $G_B,$ $\mathcal{B}, \Phi_{\mathcal{B}} \rangle$ be a transition system with respect to a classification function $h$. The problem of checking whether every transition found by the policy execution function $\Phi_{\mathcal{B}}$ induced by a given implementation $\Reach$ is correct is in $\Pi_3^p$.
\end{thm}

\begin{proof}[Proof (Sketch)]
  According to the definition of the policy execution function, every transition
  from a state $\hat{s}$ to some state $\hat{s}'$ corresponds to some
  plan $\sigma$ returned by $\Reach(\hat{s},g_B)$. Therefore, first
  one needs to check whether each plan $\sigma=\langle
  a_1,a_2,\dots,a_n\rangle$ returned by $\Reach$ given some $\hat{s}$
  and $g_B$ is correct. For that we need to check two conditions on
  the corresponding trajectories of the plan: \be[(i)]
\item for all partial trajectories $\hat{s}_0,\hat{s}_1,\dots,\hat{s}_{i-1}$ it holds that for the upcoming action $a_i$ from the plan $\sigma$, $\hat{\Phi}(\hat{s}_{i-1},a_i) \neq \emptyset$ (i.e. the action is applicable)
\item for all trajectories $\hat{s}_0,\hat{s}_1,\dots,\hat{s}_{n}$, $\hat{s}_n \models g_B$.
\ee
Checking whether these conditions hold is in $\Pi_2^p$. 

Thus, to decide whether for some state $\hat{s}$ and target $g_B$
the function $\Phi_B(\hat{s},g_B)$ does not work correctly, we can
guess $\hat{s}$ (resp. $s\ins \hat{s}$), $g_B$ and a plan $\sigma$
and verify that $\sigma \ins \Reach(\hat{s},g_B)$ and that $\sigma$ is
 not correct. As the verification is doable with an oracle for
 $\Sigma^p_2$ in polynomial time, a counterexample for correctness
 can be found in $\Sigma^p_3$;  thus the problem is in $\Pi^p_3$.
\end{proof}

The complexity is lower, if output checking of $\Reach$ has lower complexity (in particular, it drops to $\Pi^p_2$ if output checking is in co-NP).

The result for soundness of $\Reach$ is complemented with another result for completeness with respect to short (polynomial size) conformant plans that are returned by $\Reach$.
   
\begin{thm}[completeness of $\Reach$] %
Let $\mathcal{T}_h= \langle \widehat{S},$ $\widehat{S}_0,$ $G_B,$ $\mathcal{B}, \Phi_{\mathcal{B}} \rangle$ be a transition system with respect to a classification function $h$. Deciding whether for a given implementation $\Reach$, $\Phi_B$ fulfills $\hat{s}' \in \Phi_B(\hat{s})$ whenever a short conformant plan from $\hat{s}$ to $\hat{s}'$ exists in $T_h$, is in $\Pi^p_4$.
\end{thm}

\begin{proof}[Proof (Sketch)]
For a counterexample, we can guess some $\hat{s}$   
and $\hat{s}'$ (resp.\ $s\ins \hat{s}$,  $s'\ins \hat{s}'$) and some 
short plan $\sigma$ and verify that (i) $\sigma$ is a valid conformant
plan in $\mathcal{T}_h$ to reach $\hat{s}'$ from $\hat{s}$, and (ii) that
a target $g_B$ exists such that $\Reach(\hat{s},g_B)$ produces some
output.  We can verify (i) using a $\Pi^p_2$ oracle to check that
$\sigma$ is a conformant plan, and we can verify (ii) using a
$\Pi^p_3$ oracle (for all guesses of targets $g_B$ and short plans
$\sigma'$, either $g_B$ is not a target for $\hat{s}$ or $\sigma'$
is not produced by $\Reach(\hat{s},g_B)$). This establishes
membership in $\Pi^p_4$. 
\end{proof}

As in the case of soundness, the complexity drops if checking the
output of $\Reach$ is lower (in particular, to $\Pi^p_3$ if the
output checking is in co-NP).

We also restrict the plans $\sigma$ that are returned by $\Reach(\hat{s},g_B)$ to have polynomial size. This constraint would not allow for exponentially long conformant plans (even if they exist). Thus, the agent is forced under this restriction to develop targets that it can reach in polynomially many steps, and then to go on from these targets. Informally, this does not limit the capability of the agent in general. The ``long" conformant plans can be split into short plans with a modified policy and by encoding specific targets into the states. %and get plans that are ``compilable" into a reactive agent.

We denote the main goal that the reactive policy is aiming for by
$g_\infty$. Our aim is to have the capability to check whether
following the policy always results in reaching
% the states that satisfy 
some state that satisfies the main goal. That is, for each run,
i.e. sequence $\hat{s}_0,\hat{s}_1, \ldots$ such that $\hat{s}_0\in
\widehat{S}_0$ and $\hat{s}_{i+1} \in \Phi_B(\hat{s}_i)$, for all
$i\geq 0$, there is some $j\geq 0$ such that $\hat{s}_j \models
g_\infty$.
(The behavior could be easily modified to stop or to loop in any
state $\hat{s}$ that satisfies the goal.)
This way we can say whether the policy works or not.  Under the
assumptions from above, we obtain the following 

\begin{thm}
The problem of determining that the policy works is in PSPACE.
\end{thm}

\begin{proof}[Proof (Sketch)]
One needs to look at all runs 
%the trajectories 
$\hat{s}_0,\hat{s}_1, \ldots$ from every initial state $\hat{s}_0$ in
the equalized transition system and check whether each such run has
some state $\hat{s}_j$ that satisfies the main goal $g_\infty$. Given
that states have a representation in terms of fluent or state variables,
there are at most exponentially many different states. Thus to find a
counterexample, a run of at most exponential length in which
$g_\infty$ is not satisfied is sufficient. Such a run can be
nondeterministically built in polynomial space; as NPSPACE $=$ PSPACE, the
result follows. 
\end{proof}

Note that in this formulation, we have tacitly
assumed that the main goal can be established in the original system,
thus at least some trajectory from some initial state to a state
fulfilling the goal exists (this can be checked in PSPACE as well). In
a more refined version, we could define the working of a policy
relative to the fact that some abstract plan would exist that makes
$g_\infty$ true; naturally, thus may impact the complexity of the policy
checking. 

Above, we have been considering arbitrary states, targets and transitions in
the equalized transition system. In fact, for the particular
behavior, only the states that can be encountered in runs really
matter; these are the reachable states defined as follows.

\begin{defn}
A state $\hat{s}$ is \emph{reachable} from an initial state in the equalized transition system if and only if $s \in \mathcal{R}_i$ for some $i \in \mathbb{N}$ where $\mathcal{R}_i$ is defined as follows.
\beeq \ba {ll}
\mathcal{R}_0 &= \widehat{S}_0\\
\mathcal{R}_{i+1} &= \bigcup_{\hat{s}\in\mathcal{R}_i} \Phi_B(\hat{s})\\
&\dots\\
\mathcal{R}^\infty &= \bigcup_{i\geq 0} \mathcal{R}_i.
\ea
\eeeq
\end{defn}

Under the assumptions that apply to the previous results, we can state the following.
\begin{thm}
The problem of determining whether a state in an equalized transition system is reachable is in PSPACE.
\end{thm}

The notions of soundness and completeness of an outsourced planning
function $\Reach$ could be restricted to reachable states; however,
this, would not change the cost of testing these properties in general
(assuming that $\hat{s}\in \mathcal{R}$ is decidable with sufficiently
low complexity).

\subsection{Constraining equalization}

The definition of $\hat{\Phi}$ allows for certain transitions between
equalized states that don't have corresponding concrete transitions in
the original transition system. However, the aim of defining such an
equalized transition system is not to introduce new features, but to
keep the structure of the original transition system and discard the
unnecessary parts with respect to the policy. Therefore, one needs to
give further restrictions on the transitions of the equalized
transition system, in order to obtain the main objective.

Let us consider the following condition
\beq
\hat{s}' \in \hat{\Phi}(\hat{s},a) \Leftrightarrow \forall s'\in\hat{s}',~ \exists s \in \hat{s}:~ s' \in \Phi(s,a) 
\eeq {eqn:cond}
This condition ensures that a transition between two states $\hat{s}_1, \hat{s}_2$ in the equalized transition system represents that any state in $\hat{s}_2$ has a transition from some state in $\hat{s}_1$. 
An equalization is called \emph{proper} if condition \eqref{eqn:cond} is satisfied.

\begin{thm}
Let $\mathcal{T}_h{=} \langle \widehat{S}, \widehat{S}_0,G_B, \mathcal{B}, \Phi_{\mathcal{B}} \rangle$ be a transition system with respect to a classification function $h$. Let $\hat{\Phi}$ be the transition function that the policy execution function $\Phi_{\mathcal{B}}$ is based on. The problem of checking whether $\hat{\Phi}$ is proper is in $\Pi_2^p$.
\end{thm}

\begin{proof}[Proof (sketch)]
As a counterexample, one needs to guess $\hat{s},a,$ $\hat{s}'\ins \hat{\Phi}(\hat{s},a)$ and $s' \ins \hat{s}'$ such that no $s \ins \hat{s}$ has $s' \ins \Phi(s,a)$.
\end{proof}

The results in Theorems 1-5
are all complemented by lower bounds for realistic realizations of
the parameters (notably, for typical action languages such as
fragments of ${\cal C}$).

The following proposition is based on the assumption that the transition function $\hat{\Phi}$ satisfies condition \eqref{eqn:cond}.

\begin{prop}[soundness]
Let $\mathcal{T}_h{=} \langle \widehat{S}, \widehat{S}_0,G_B, \mathcal{B}, \Phi_{\mathcal{B}} \rangle$ be a transition system with respect to a classification function $h$. Let $\hat{s}_1,\hat{s}_2 \in \widehat{S}$ be equalized states that are reachable from some initial states, and $\hat{s}_2 \in \Phi_B(\hat{s}_1)$. Then for any concrete state $s_2 \in \hat{s}_2$ there is a concrete state $s_1 \in \hat{s}_1$ such that $s_1 \rightarrow^\sigma s_2$ for some action sequence $\sigma$, in the original transition system.
\end{prop}

\begin{proof}
  For equalized states $\hat{s}_1,\hat{s}_2$, having $\hat{s}_2 \ins
  \Phi_B(\hat{s}_1)$ means that $\hat{s}_2$ satisfies a goal condition
  that is determined at $\hat{s}_1$, and is reachable via executing
  some plan $\sigma$. With the assumption that \eqref{eqn:cond} holds,
  we can apply backwards tracking from any state $s_2 \in
  \hat{s}_2$ following the transitions $\Phi$ corresponding to the
  actions in the plan $\sigma$ backwards. In the end, we can find a
  concrete state $s_1 \ins \hat{s}_1$ from which one can reach
  the state $s_2 \ins \hat{s}_2$ by applying the plan $\sigma$
  in the original transition system.
\end{proof}

Thus, we can conclude the following corollary, with the requirement of only having initial states clustered into the equalized initial states (i.e. no ``non-initial" state is mapped to an initial equalized state). Technically, it should hold that $\forall s \in S_0: h^{-1}(h(s)) \subseteq S_0$.  

\begin{cor}
If there is a trajectory in the equalized transition system with initial state clustering from an equalized initial state $\hat{s}_0$ to $g_\infty$, then it is possible to find a trajectory in the original transition system from some concrete initial state $s_0 \in \hat{s}_0$ to $g_\infty$.
\end{cor}

We want to be able to study the reactive policy through the equalized
transition system. In case the policy does not work as expected, there
should be trajectories that shows the reason of the failure. Knowing
that any such trajectory found in the equalized transition system
exists in the original transition system is enough to conclude that
the policy indeed does not work.

Current assumptions can not avoid the case where a plan $\sigma$ returned by $\Reach$ on the equalized transition system does not have a corresponding trajectory in the original transition system. Therefore, we consider an additional condition as
\beq
\hat{s}' \in \hat{\Phi}(\hat{s},a) \Leftrightarrow \forall s\in
\hat{s},~ \exists s'\in \hat{s}': s' \in \Phi(s,a) 
\eeq {eqn:cond2}
that strengthens the properness condition \eqref{eqn:cond}. Under this condition, every plan
returned by $\Reach$ can be successfully executed in the original
transition system and will establish the target $g_B$. However, still we
may lose trajectories of the original system as by clustering states 
they might not turn into conformant plans. Then one would need to modify the description of determining targets, i.e. the set of targets $G_B$ and the function $\mathcal{B}$.
%\todo[inline]{the other direction, does it bother us}
\begin{exmp}

\begin{figure}[h!]
\centering
\subfigure[][Successor states from $\Phi_B(\hat{s}_1)$\label{fig:flow}]{\includegraphics[scale=0.8]{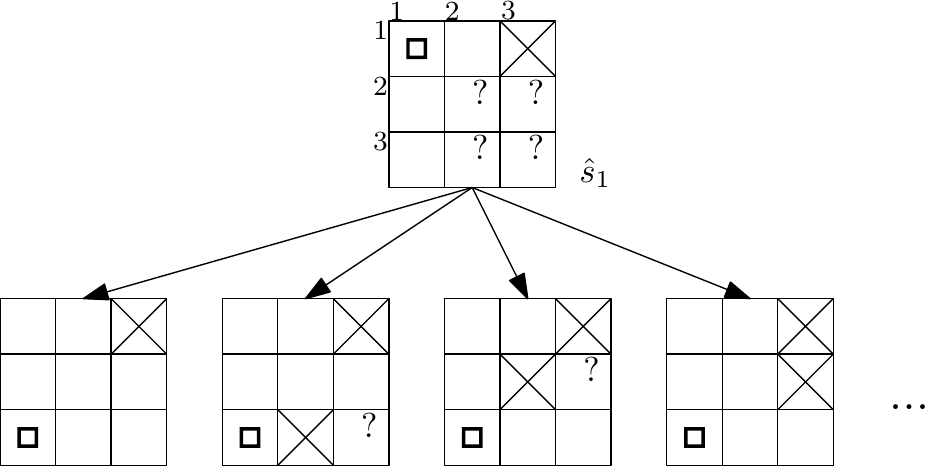}}
\subfigure[][Disregarding irrelevant states\label{fig:redundant}]{\includegraphics[scale=0.8]{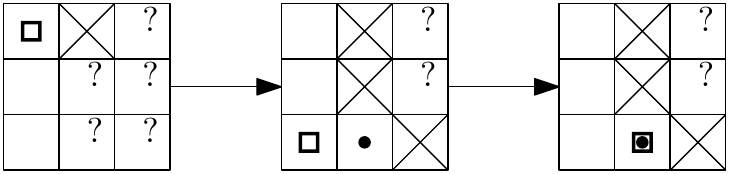}}
\caption{Parts of an equalized transition system}
\end{figure}

Remember the environment and the policy described in the running example, and consider the scenario shown in Figure~\ref{fig:flow}. It shows a part of the equalized transition system constructed according to the policy. 
The states that are not distinguishable due to the partial observability are clustered into the same state.

The policy is applied according to current observations, and the successor states show the possible resulting states. The aim of the policy is to have the agent move to the farthest reachable point, which for $\hat{s}_1$ is $(3,1)$. As expected, there can be several states that satisfy the target $g_B{=}robotAt(3,1)$. The successor states of $\Phi_B(\hat{s}_1)$ is determined by $\Res(\hat{s}_1,\sigma)$ computing the possible resulting states after executing the plan $\sigma$ returned by $\Reach(\hat{s}_1,g_B)$. Considering that the agent will gain knowledge about the environment while moving, there are several possibilities for the resulting state.

Notice that this notion of a transition system can help in reducing the number of states, due to the fact that it is able to disregard states with information on fluents that does not have any effect on the system's behavior. For example, Figure~\ref{fig:redundant} shows a case where the unknown parts behind the obstacles are not relevant to the agent's behavior, i.e. the person can be found nonetheless.%without knowing what's behind the obstacles. %Building a transition system through the observed states helps in not considering all possible values all fluents can take even though they are not relevant to the current behavior. 
\label{ex:2}
\end{exmp}

\section{Relation with Action Languages}

In this section, we describe how our definition of a higher-level transition system that models the behavior can fit into the action languages. Given a program defined by an action language and its respective (original) transition system, we now describe how to model this program following a reactive policy and how to construct the corresponding equalized transition system according to the policy.

\subsection{Classifying the state space} 

The approach to classify the (original) state space relies on defining a function that classifies the states. There are at least two kinds of such classification; one can classify the states depending on whether they give the same values for certain fluents and omit the knowledge of the values of the remaining fluents, or one can introduce a new set of fluents and classify the states depending on whether they give the same values for the new fluents:
\bi
\item Type 1: Extend the set of truth values by ${\bf V}'={\bf V} \cup \{\mathsf{u}\}$, where $\mathsf{u}$ denotes the value to be \emph{unknown}. Extend the value function by $V': {\bf F} \times S \rightarrow \textbf{V}'$. Then, consider a new set of groups of states, $\widehat{S}=\{\hat{s}_1,\dots,\hat{s}_n\}$, where a group state $\hat{s}_i$ contains all the states $s\in S$ that give the same values for all $p \in \textbf{F}$, i.e. $\widehat{S} = \{\hat{s} ~|~ \forall d,e \in S,~ d,e \in \hat{s} {\iff}\forall p\in \textbf{F}: V'(p,d){=}V'(p,e)~ \}$. The value function for the new group of states is $\widehat{V}:\textbf{F} \times \widehat{S} \rightarrow \textbf{V}'$.

\item Type 2: Consider a new set of (auxiliary) fluent names $\textbf{F}_a$, where each fluent $p \in \textbf{F}_a$ is \emph{related} with some fluents of $\textbf{F}$. The relation can be shown with a mapping $\Delta:2^{\textbf{F} \times \textbf{V}} \rightarrow \textbf{F}_a \times \textbf{V}$. %Extend the value function by $V': \textbf{F} \cup \textbf{F}_a \times S \rightarrow \textbf{V}$. 
Then, consider a new set of groups of states, $\widehat{S}=\{\hat{s}_1,\dots,\hat{s}_n\}$, where a group state $\hat{s}_i$ contains all the states $s\in S$ that give the same values for all $p \in \textbf{F}_{a}$, i.e. $\widehat{S} = \{\hat{s} ~|~ \forall d,e \in S,~ d,e \in \hat{s} {\iff}\forall p\in \textbf{F}_a: V(p,d){=}V(p,e)~ \}$. The value function for the new group of states is $\widehat{V} : \textbf{F}_a \times \widehat{S} \rightarrow \textbf{V}$.
\ei

We can consider the states in the same classification to have the same \emph{profile}, and the classification function $h$ as a membership function that assigns the states into groups. 

\paragraph{Remarks:} (1) In Type 1, introducing the value \emph{unknown} for the fluents allows for describing sensing actions and knowing the true value of a fluent at a later state. Also, one needs to give constraints for a fluent to have the \emph{unknown} value. e.g. it can't be the case that a fluent related to a grid cell is unknown while the robot is able to observe it.\\
(2) In Type 2, one needs to modify the action descriptions according to the newly defined fluents and define \emph{abstract actions}. However, in Type 1, the modification of the action definitions is not necessary, assuming that the actions are defined in a way that the fluents that are used when determining an action always have \emph{known} values.

Once a set of equalized states is constructed according to the classification function, one needs to define the reactive policy to determine the transitions. Next, we describe how a policy can be defined from an abstract point of view, through a \emph{target language} which figures out the targets and helps in determining the course of actions, and show how the transitions are constructed.

\subsection{Defining a target language} 

Let $\widehat{\textbf{F}}$ denote the set of fluents that the equalized transition system is built upon. Let $\mathcal{F}(\widehat{\textbf{F}})$ denote the set of formulas in an abstract language that can be constructed over %the fluents of 
$\widehat{\textbf{F}}$.

%A reactive policy can be described through a \emph{target language} which helps in determining the course of actions at a state. 
We consider a declarative way of finding targets. 
Let $\mathcal{F}_\mathcal{B}(\widehat{\textbf{F}}){\subseteq} \mathcal{F}(\widehat{\textbf{F}})$ be the set of formulas that describe target determination. Let $\mathcal{F}_{G_B}(\widehat{\textbf{F}}) {\subseteq} \mathcal{F}(\widehat{\textbf{F}})$ denote the set of possible targets that can be determined via the evaluation of the formulas $\mathcal{F}_\mathcal{B}(\widehat{\textbf{F}})$ over the related fluents in the equalized states. 

Notice that separation of the target determining formulas $\mathcal{F}_\mathcal{B}(\widehat{\textbf{F}})$ and the targets $\mathcal{F}_{G_B}(\widehat{\textbf{F}})$ is to allow for outsourced planners that are able to understand simple target formulas. These planners do not need to know about the target language in order to find plans. However, if one is able to use planners that are powerful enough, then the target language can be given as input to the planner, so that the planner determines the target and finds the corresponding plan.

To define a relation between $\mathcal{F}_\mathcal{B}(\widehat{\textbf{F}})$ and $\mathcal{F}_{G_B}(\hatbf{F})$, we introduce some placeholder fluents. Let $\mathcal{F}_\mathcal{B}(\hatbf{F})=\{f_1,\dots,f_n\}$ be the set of target formulas. Consider a new set of fluents $\hatbf{F}_\mathcal{B}=\{p_{f_1},\dots,p_{f_n}\}$ where each of the formulas in $\mathcal{F}_\mathcal{B}$ is represented by some fluent. The value of a fluent depends on whether its respective formula is satisfied or not, i.e. for a state $s$, $s \models f {\iff} V(p_f,s)=\textsf{t}$. Now consider a mapping $\mathcal{M}: 2^{\hatbf{F}_\mathcal{B}} \rightarrow 2^{\mathcal{F}_{G_B}(\hatbf{F})}$ where 
\beeq
\mathcal{M}(\{p_{f_1},p_{f_2},\dots,p_{f_m}\})=\{g_1,\dots,g_r\}, m \leqs n \hbox{ and } r\geqs 1
\eeeq
means that if there is a state $s$ such that $s {\models} f_i$, $1 \leqs i \leqs m$ and $s \nmodels f$ for the remaining formulas $f\ins \mathcal{F}_\mathcal{B}(\widehat{\textbf{F}})\backslash\{f_1,\dots,f_m\}$, then in the successor state $s'$ of $s$, $s' \models g_i$ for some $ 1 \leqs i \leqs r$, should hold. We consider the output of $M$ to be a set of targets in order to represent the possibility of nondeterminism in choosing a target.

\subsection{Transition between states} 

The transition for the equalized transition system can be denoted with $\widehat{R} {\subseteq} \widehat{S}{\times} \widehat{S}$, where $\widehat{R}$ corresponds to the policy execution function $\Phi_B$ that uses (a) the target language to determine targets, (b) an outsourced planner (corresponding to the function $\Reach$) to find conformant plans and (c) the computation of executing the plans (corresponding to the function $\Res$). The outsourced planner finds a sequence of actions $\sigma \in 2^\mathcal{A}$ from an equalized state $\hat{s}$ to one of its determined targets $g_B$. Then the successor equalized states are computed by executing the plan from $\hat{s}$. Transition $\widehat{R}$ shows the resulting states after applying the policy.

%%%%%%%%%%%%%%%%%%%%%%%blocksworld example begin%%%%%%%%%%%%%%%%%%%%%%%
\begin{exmp}
Let us consider a simple blocksworld example 
where a policy (of two phases) is defined as follows:
\bi
\item if at phase 1 and not all the blocks are on the table, move one free block on a stack with highest number of blocks to the table.
\item if all the blocks are on the table, move to phase 2.
\item if at phase 2 and not all the blocks are on top of each other, move one of the free blocks on the table on top of the stack with more than one block (if exists any, otherwise move the block on top of some block).
\ei

Since the policy does not take blocks' labels into consideration, a classification can be of the following form for $n$ number of blocks: We introduce an $n$-tuple $\langle b_1,\dots,b_n \rangle$ to denote equalized states such that for $i \leq n$, $b_i$ would represent the number of stacks that have $i$ blocks. For example, for $4$ blocks, a state $\langle 1,0,1,0\rangle$ where $b_1=1,b_2=0,b_3=1,b_4=0$ would represent all the states in the original transition system with the profile ``contains a stack of 1 block and a stack of 3 blocks". Notice that in the original transition system for 4 labeled blocks, there are 24 possible states that have this profile and if the blocks need to be in order, then there are 4 possible states.

\begin{figure}[t]
\centering
\includegraphics[scale=0.8]{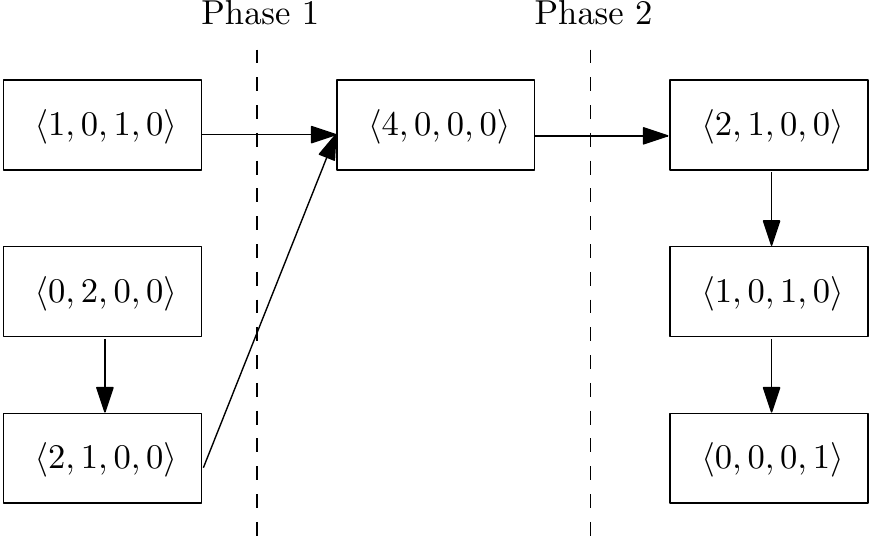}
\caption{Eq. transition system of blocksworld ($n\eqs4$)}
\label{fig:blocks}
\end{figure}

Figure~\ref{fig:blocks} demonstrates the corresponding equalized transition system for the case of $4$ blocks. The equalized transition system for this example is in the following form:
\bi
\item $\widehat{S}$ is the set of equalized states according to the abstraction as described above.
\item $\widehat{S}_0 \in \widehat{S}$ is the initial equalized states (all elements of $\widehat{S}$ except $\langle 0,\dots,0,1\rangle$).
\item $G_B=\widehat{S}$, since the policy is related with all the blocks, it can determine targets as the whole states.
\item $\mathcal{B}: \widehat{S} \rightarrow \widehat{S}$ is the target function.
\item $\Phi_B: \widehat{S} \rightarrow \widehat{S}$ is the policy execution function, returning the resulting successor state after applying one action desired by the behavior, shown as in Figure~\ref{fig:blocks}.
\ei

\label{ex:blocks}
\end{exmp}

\section{Application on Action Language $\mathcal{C}$}

In this section, we describe how one can construct an equalized transition system for a reactive system that is represented using the action language $\mathcal{C}$ \cite{c98}. First, we give some background information about the language $\mathcal{C}$, then move on to the application of our definitions.

\paragraph{Syntax} A \emph{formula} is a propositional combination of fluents.

Given a propositional action signature $\langle \{\mathsf{f,t}\},\textbf{F},\textbf{E}\}$,
whose set \textbf{E} of elementary action names is disjoint from \textbf{F},
\noindent an \emph{action description} is a set of expressions of the following forms:
\bi
\item \emph{static laws:}
\beq
\textbf{caused } F \textbf{ if } G,
\eeq {eqn:c_static}
where $F$ and $G$ are formulas that do not contain elementary actions;
\item \emph{dynamic laws:}
\beq
\textbf{caused } F \textbf{ if } G \textbf{ after } U,
\eeq {eqn:c_dynamic}
where $F$ and $G$ are as above, and $U$ is a formula.
\ei

\paragraph{Semantics} The transition system $\langle S,V,R \rangle$ \emph{described} by an action description $D$ is defined as follows:
\be [(i)]
\item $S$ is the set of all interpretations $s$ of {\bf F} such that, for every static law \eqref{eqn:c_static} $s$ satisfies $F$ if $s$ satisfies $G$,
\item $V(P,s)=s(P),$ i.e. identify $s$ with $V(P,s),$
\item $R$ is the set of all triples $\langle s,A,s' \rangle$, $A \subseteq \textbf{E}$, such that $s'$ is the only interpretation of {\bf F} which satisfies the heads of all
\bi
\item static laws \eqref{eqn:c_static} in $D$ for which $s'$ satisfies $G$, and
\item dynamic laws \eqref{eqn:c_dynamic} in $D$ for which $s'$ satisfies $G$ and $s \cup A$ satisfies $U$.
\ei
\ee

We focus on a fragment of the language $\mathcal{C}$ where the heads of the static and dynamic laws only consist of literals. This restriction on the laws reduces the cost of evaluating the transitions $\langle s, A, s' \rangle \ins R$ to polynomial time. Thus, we match the
   conditions on complexity from above. Furthermore, by well-known
   results on the complexity of action language ${\cal C}$ \cite{turner2002polynomial,DBLP:journals/tocl/EiterFLPP04}  all the results in Theorems
   1-5 can be turned into completeness results already for this fragment.

\subsection{Defining a policy}

Let $\hatbf{F}$ be the set of fluents that are relevant to the policy. The target language is defined explicitly via static laws using the fluents in $\hatbf{F}$, denoted $\mathcal{F}_{\mathcal{B}}(\hatbf{F})$, where a target is determined by the evaluation of these formulas in a state. 

\begin{exmp}
An example of a target language for the running example uses causal laws from $\mathcal{C}$:
\beeq \ba l
\caused{target(X1,Y1)}{}\\ \hfill \mi{robotAt(X,Y) ~\wedge~ farthest(X,Y,X1,Y1)} \\ \hfill \mi{~\wedge~ not~ personDetected.}\\
\caused{personDetected}{personDetected(X,Y).}\\
\caused{targetPerson(X,Y)}{personDetected(X,Y).}\\
\caused{personFound}{personDetected(X,Y)}\\ \hfill \mi{~\wedge~ robotAt(X,Y).}
\ea \eeeq %{eqn:policy}
where $\mathcal{F}_{G_B}(\hatbf{F})$ consists of all atoms
$\mi{target(X,Y)}$ and $\mi{targetPerson(X,Y)}$ for $1\leqs \mi{X} \leqs n, 1\leqs \mi{Y} \leqs n$.

The target of a state according to the
policy is computed through joint evaluation of these 
%formulas 
causal laws over the state with the \emph{known} fluents about the agent's location and the reachable points. Then, the outsourced planner may take as input the agent's current location and the target location, to find a plan to reach the target.
\end{exmp}

\subsubsection{Equalized transition system} The equalized transition system $\langle \widehat{S},\widehat{V},\widehat{R} \rangle$ that describes the policy is defined as follows:
\be [(i)]
\item $\widehat{S}$ is the set of all interpretations of $\hatbf{F}$ such that, for every static law \eqref{eqn:c_static} $\hat{s}$ satisfies $F$ if $\hat{s}$ satisfies $G$,
\item $\widehat{V}(P,\hat{s})=\hat{s}(P)$, where $P \in \hatbf{F}$,
\item $\widehat{R} \subseteq \widehat{S} \times \widehat{S}$ is the set of all $\langle \hat{s}, \hat{s}' \rangle$ such that
\be
\item for every $s'\in \hat{s}'$ there is a trajectory from some $s \in \hat{s}$ of the form $s,A_1,s_1,\dots,A_n,s'$ in the original transition system;
\item for static laws $f_1,f_2,\dots,f_m \in \mathcal{F}_\mathcal{B}(\hatbf{F})$ for which $\hat{s}$ satisfies the body, it holds that $\hat{s}' \models g$ for some $g\in \mathcal{M}(p_{f_1},\dots,p_{f_m}).$
\ee
\ee

Notice that in the definition of the transition relation $\widehat{R}$ in (iii) there is no description of (a) how a trajectory is computed or (b) how a target is determined. This gives flexibility on the implementation of these components. 

Other languages can be similarly used to describe the equalized transition system, as long as they are powerful enough to express the concepts from the previous section.

\section{Related Work}

There are works being conducted on the verification of GOLOG programs \cite{golog97}, a family of high-level action programming languages defined on top of action theories expressed in the situation calculus. 
The method of verifying properties of non-terminal processes are sound, but do not have the guarantee of termination due to the verification problem being undecidable \cite{giacomo97,classen08}. By resorting to action formalisms based on description logic, decidability can be achieved \cite{verifygolog13}. %In a later work, the authors extend these results considering branching-time logics \cite{zarriess14}.

Verifying temporal properties of dynamic systems in the context of data management is studied by \cite{calvanese2013verification} in the presence of description logic knowledge bases. However, target establishment and planning components are not considered in these works, and they do not address real life environment settings.

The logical framework for agent theory developed by Rao and Georgeff \shortcite{bdi91} is based on beliefs, desires and intentions, in which agents
are viewed as being rational and acting in accordance with their beliefs and goals. %The agent has beliefs about itself and its environment, desires (or goals) representing its long-term aims, and intentions describing its immediate goals. %Their main concerns was the question of how an agent's beliefs about the future affect its desires and intentions.
There are many different agent programming languages and platforms based on the BDI approach. Some works carried out on verifying properties of agents represented in these languages, such as \cite{bordini06,dennis12}. These approaches consider very complex architectures that even contain a plan library where plans are matched with the intentions or the agent's state and manipulate the intentions.

\paragraph{Synthesizing and Verifying Plans} 

There have been various works on synthesizing plans via symbolic model checking techniques by Cimatti et al.\ \shortcite{cimattistrong98,cimattiobdd98}, Bertoli et al.\ \shortcite{bertoli06}.  
These approaches are able to solve difficult planning problems 
like strong planning and strong cyclic planning. 

Son and Baral \shortcite{baral01} extend the action description language by allowing sensing actions and allow to query conditional plans.
These conditional plans are general plans that consist of sensing actions and conditional statements. %However, in large environments conditioning all the possibilities would be inapplicable.

These works address a different problem then ours. When nondeterminism and partial observability are taken into account, finding a plan that satisfies the desired results in the environment is highly demanding. We consider a much less ambitious approach where given a behavior, we aim to check whether or not this behavior gives the desired results in the environment. However, our framework is capable emulating the plans found by these works.

\paragraph{Execution Monitoring} There are logic-based monitoring frameworks that monitor the plan execution and recover the plans in case of failure. The approaches that are studied are replanning \cite{de1998execution}, backtracking to the point of failure and continuing from there \cite{soutchanski2003high}, or diagnosing the failure and recovering from the failure situation \cite{fichtner2003intelligent,eiter2007logic}.

These works consider the execution of a given plan, while we consider a given reactive policy that determines targets and use (online) planning to reach these targets.

\section{Conclusion and Future Work}

In this paper, we described a high-level representation that models
reactive behaviors, and integrates target development and online
planning capabilities. Flexibility in these components does not bound
one to only use action languages, but allows for the use of other
formalizations as well. 
For future work, one could
imagine targets to depend on further parameters or to incorporate learning from experience in the framework. It is also possible to use other plans, e.g. short conditional plans, in the planner component.

The long-term goal of this work is to check and verify properties of
the reactive policies for action languages. In order to solve these problems practically,
it is necessary to use techniques from model checking, such as
abstraction, compositional reasoning and parameterization. Also, the use of temporal logic formulas is needed to express complex goals such as properties of the policies. Our main target is to work with action languages, and to incorporate their syntax and semantics with such model checking techniques. The general structure of our framework allows one to focus on action languages, and to investigate how to merge these techniques.

% \textit{References and End of Paper}\\
\bibliography{ref.bib}
\bibliographystyle{aaai}

\end{document}